\documentclass{article}

\usepackage[preprint]{neurips_2025}

\usepackage[utf8]{inputenc} 
\usepackage[T1]{fontenc}    
\usepackage{hyperref}       
\usepackage{url}            
\usepackage{booktabs}       
\usepackage{amsfonts}       
\usepackage{nicefrac}       
\usepackage{microtype}      

\usepackage{amsthm}
\newtheorem{theorem}{Theorem}[section]
\newtheorem{lemma}[theorem]{Lemma}

\usepackage{graphicx}       
\usepackage{subcaption}     
\usepackage{float}          
\usepackage{algorithm}
\usepackage{algpseudocode}

\usepackage{siunitx}        
\usepackage[table]{xcolor}
\definecolor{darkred}{RGB}{139,0,0}

\usepackage{amsmath}

\usepackage{cleveref}       
\usepackage{multirow}
\title{Diagonal Batching Unlocks Parallelism in Recurrent Memory Transformers for Long Contexts}

\author{
Danil Sivtsov\textsuperscript{1,2} \quad
Ivan Rodkin\textsuperscript{3,4} \quad
Gleb Kuzmin\textsuperscript{1,5} \quad
Yuri Kuratov\textsuperscript{1,3} \quad
Ivan Oseledets\textsuperscript{1,2} \\
\textsuperscript{1}AIRI, Moscow, Russia \\
\textsuperscript{2}Skoltech, Moscow, Russia \\
\textsuperscript{3}Neural Networks and Deep Learning Lab, MIPT, Dolgoprudny, Russia \\
\textsuperscript{4}MBZUAI, Abu Dhabi, UAE \\
\textsuperscript{5}FRC CSC RAS, Moscow, Russia \\
}

\begin{document}

\maketitle

\begin{abstract}
Transformer models struggle with long-context inference due to their quadratic time and linear memory complexity. Recurrent Memory Transformers (RMTs) offer a solution by reducing the asymptotic cost to linear time and constant memory usage. However, their memory update mechanism leads to sequential execution, causing a performance bottleneck.

We introduce Diagonal Batching, a scheduling scheme that unlocks parallelism across segments in RMTs while preserving exact recurrence. This approach eliminates the sequential constraint, enabling efficient GPU inference even for single long-context inputs without complex batching and pipelining techniques. Because the technique is purely a run-time computation reordering, existing RMT models adopt it with no retraining.

Applied to a LLaMA-1B ARMT model, Diagonal Batching yields a 3.3x speedup over standard full-attention LLaMA-1B and a 1.8x speedup over the sequential RMT implementation on 131,072-token sequences.
By removing sequential bottleneck, Diagonal Batching reduces inference cost and latency, thereby strengthening RMTs as a practical solution for real-world, long-context applications.
\end{abstract}

\begin{figure}[H]
  \centering
  \includegraphics[width=0.49\textwidth]{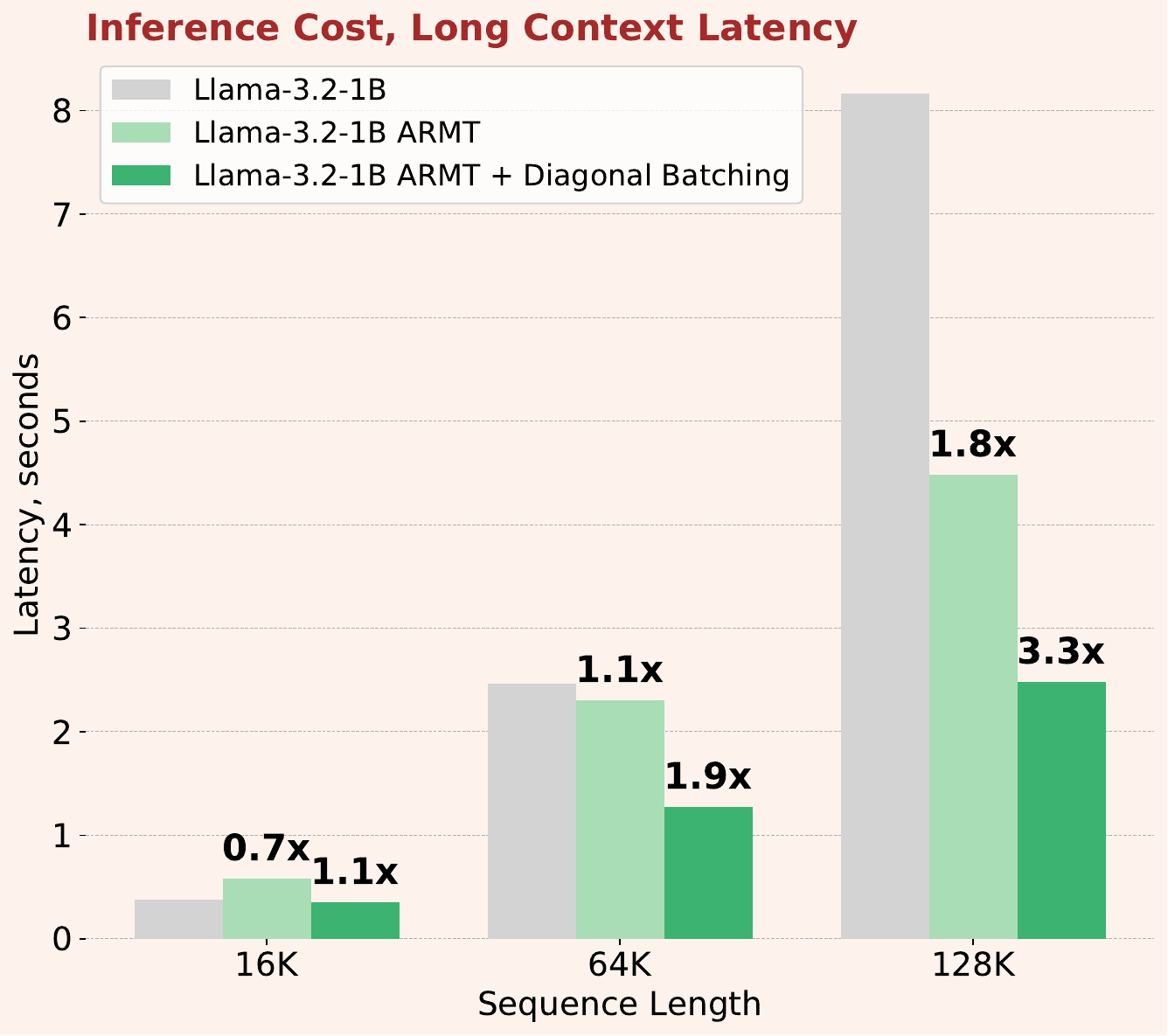}
  \hfill
  \includegraphics[width=0.49\textwidth]{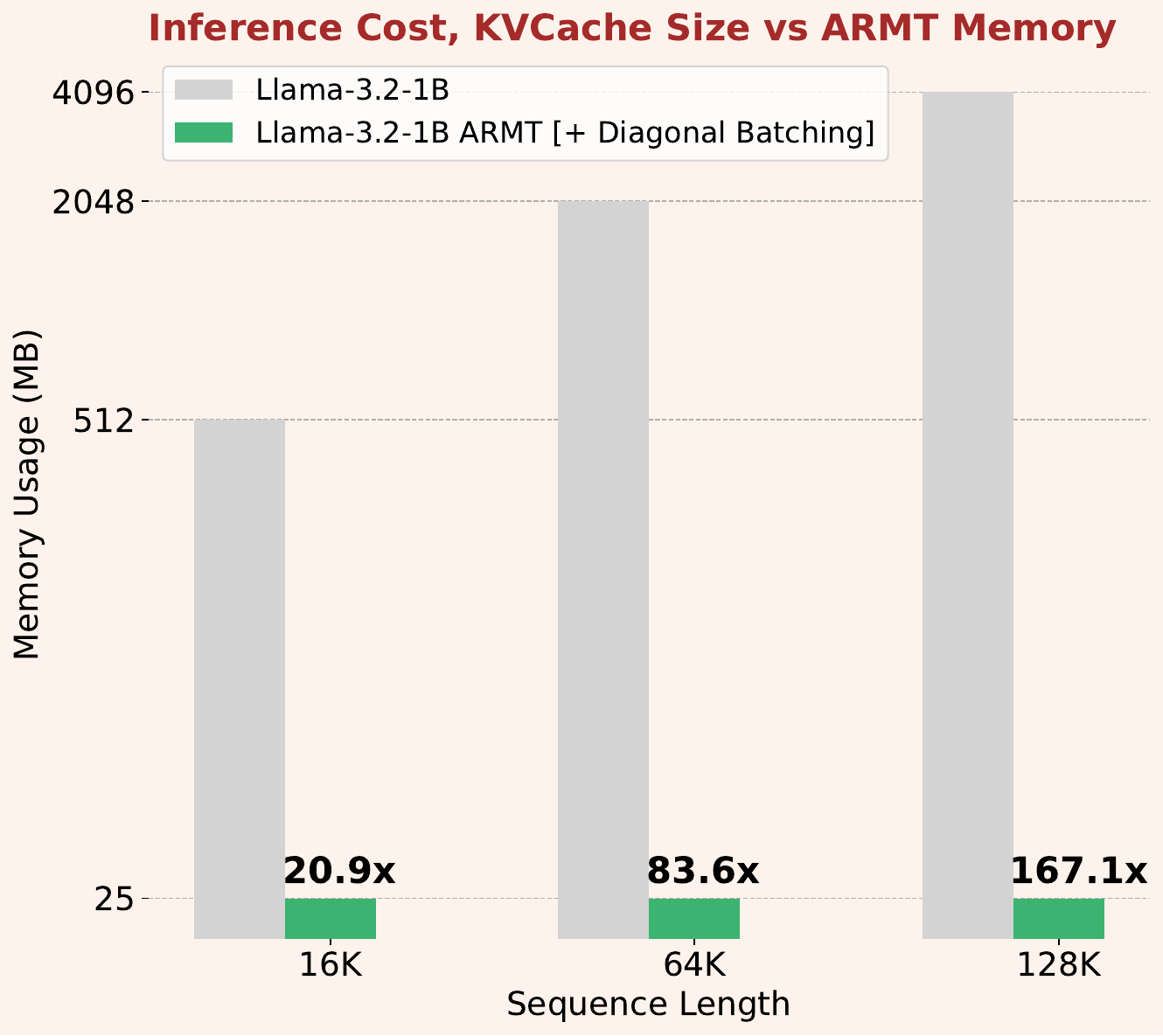}
  \caption{\textbf{Diagonal Batching enables the Recurrent Memory Transformers (ARMT) to process 128k tokens sequences 3.3x faster than the LLama-3.2-1B model, with 167.1x memory savings.} These results were obtained using an A100 GPU, and the segment size for the ARMT was set to 1,024 tokens. }
\end{figure}

\section{Introduction}
\label{sec:introduction}
Transformer-based language models have not only revolutionized natural language processing (NLP)~\citep{vaswani2017attention,devlin2019bert,radford2019gpt2}, but also catalyzed the development of intelligent agents that can solve complex, multi-step problems in various domains by scaling up to large language models (LLMs)~\citep{openai2023gpt4,reid2024gemini, dubey2024llama3}. However, these transformer-based models have quadratic time complexity and a linear memory footprint with respect to the length of the input sequence. Consequently, real-world applications are limited by the context window size of standard transformers that can fit within hardware constraints.

From an engineering perspective, numerous optimizations have been proposed to improve attention efficiency and manage GPU memory more effectively.  Optimized attention kernels, such as FlashAttention~\citep{dao2022flashattention,dao2023flashattention2} and the xFormers library~\citep{xFormers2022} focus on reducing memory access overhead and maximizing throughput. Memory-saving attention modifications like Multi-Query Attention (MQA)~\citep{shazeer2019mqa}, Grouped Query Attention (GQA)~\citep{ainslie2023gqa}, and Multi-head Latent Attention (MLA)~\citep{liu2024deepseekv2} lower GPU RAM usage by sharing and optimizing KV-cache. For distributed long-context training, methods like Ring Attention~\citep{ringattention} and Microsoft DeepSpeed's Ulysses~\citep{jacobs2023deepspeedulysses} partition sequence data across multiple devices to scale beyond single-GPU memory limits.

Along with these engineering optimizations, alternative architectures to the standard Transformer have been explored. Recently, linear recurrent models, such as S4~\citep{gu2021s4}, RWKV~\citep{peng-etal-2023-rwkv}, RetNet~\citep{sun2023retentive}, and Mamba~\citep{mamba,dao2024mamba2} have replaced the softmax attention with alternative read-write operations. These models offer efficient parallel training, like transformers, and require constant memory during inference, like RNNs. However, these approaches often suffer from reduced memory capacity~\citep{jelassi2024repeat} and decreased accuracy in read-write operations~\citep{rodkin2024associative}.
Furthermore, both state-space models and Transformers face theoretical limits, such as the $\text{TC}^0$ complexity bound on the class of functions computable in a single forward pass~\citep{merrill2024illusion,strobl2024formal}, constraining their expressivity despite massive parallelism.



Memory-augmented models~\citep{weston2014memory,sukhbaatar2015endtoend}, especially memory-augmented transformers with segment-level recurrence~\citep{dai2019transformerxl,rae2019compressive,bulatov2022recurrent,hutchins2022blockrecurrent} offer an alternative approach by compressing history into fixed-size memory states and propagating them across segments. In Recurrent Memory Transformers (RMT)~\citep{bulatov2022recurrent}, special memory tokens carry state between segments, and each Transformer block acts as a recurrent cell. This approach reduces inference complexity to linear time and constant memory, supporting arbitrarily long contexts~\citep{bulatov2023scaling}. However, the recurrent nature of RMT makes it not fully parallelizable; all subsequent layers have recurrent dependencies, and all segments must be processed sequentially.



Parallel Recurrent Memory Transformers (PRMTs)~\citep{rodkin2024associative} are a broader class of architectures in which each layer maintains its own memory state. PRMTs localize recurrence within layers and eliminate all inter-layer memory flow. The Associative Recurrent Memory Transformer (ARMT)~\citep{rodkin2024associative} belongs to this family and demonstrates exceptional scalability. It maintains high quality on sequences of up to 50 million tokens, which is far beyond the capacity of RMT and Mamba~\citep{rodkin2024associative,babilong}. Models such as RWKV, Mamba, and other linear-recurrent architectures can also be considered members of the PRMT family due to their layer-level memory design. In practice, however, these methods only exploit parallelism within individual segments. This parallelism is limited by RAM and compute bounds. Therefore, when processing extremely long sequences, these methods fall back to processing sequential segments, or even to token-level recurrence. This leaves true inter-segment parallelism unaddressed.

In this work, we introduce \emph{Diagonal Batching}, a scheduling scheme that unlocks inter-segment parallelism in PRMTs inference without altering their exact recurrence. By reorganizing the 2D grid of layer and segment computations into independent "diagonals" our method enables concurrent execution of up to N\_Layers operations per GPU kernel launch. Diagonal Batching fully encapsulates transformer block computations across segments, thus \emph{eliminating the layer- and segment-level synchronization barriers} present in previous RMT implementations.

We implement diagonal batching in the ARMT framework and evaluate its performance on a LLaMA-1B, 3B, and 8B models with sequence lengths up to 131{,}072 tokens on an NVIDIA A100/H100 GPUs. Our experiments demonstrate a $3.3\times$ speedup over standard full-attention inference and a $1.8\times$ improvement relative to a sequential ARMT baseline for 1B models. These results demonstrate that diagonal batching is a practical solution for exact, linear-time inference on extremely long contexts. Diagonal Batching code and experiments are publicly available.\footnote{\href{https://github.com/svtdanny/diagonal-batching}{github.com/svtdanny/diagonal-batching}}

Our contributions are:
\begin{itemize}
    \item We identify the key bottlenecks in existing implementations of RMTs and PRMTs, that limit efficient long‑context inference.
    \item We introduce a novel \emph{Diagonal Batching} technique that maximizes GPU utilization, preserves exact recurrence, and efficiently handles recurrent dependencies in PRMTs, enabling practical parallel execution.
    \item We empirically demonstrate that our diagonal batching method allows RMTs to achieve long-context scaling performance matching to the batch size scaling of their underlying transformer architectures.
    \item Our approach utilizes GPU with one long context request at a time, simplifying load balancing for production deployment.
\end{itemize}


\section{Background}
\label{sec:background}

\subsection{Recurrent Memory Transformers}
\paragraph{Recurrent Memory Transformer} (RMT) extends standard Transformer architectures by introducing segment-level recurrence. Specifically, the hidden representations corresponding to a segment $s$ are conditioned on a recurrent state $M$—referred to as the \textit{memory}—propagated from the previous segment $s-1$.

In the original RMT formulation, the memory state is implemented as a sequence of embeddings (\autoref{fig:diagonal_visualization}, left). The memory update mechanism can be formally expressed as:

\begin{gather}
[\_, \_, M_{s}] = \text{Transformer}([M_{s-1}, H_{s-1}, M_{s-1}]),
\end{gather}

where $M_s$ denotes the memory state associated with segment $s$, and $H_{s-1}$ represents the input embeddings from segment $s-1$. The square brackets indicate concatenation of the input sequences.

\paragraph{Associative Recurrent Memory Transformer} (ARMT) introduces a parallel memory mechanism designed to support a hierarchical memory structure. Unlike the original RMT, ARMT maintains distinct memory states across different layers. This design facilitates a more expressive memory representation by allowing each layer to store and update its own memory.

The memory update rule in ARMT is formulated as follows:

\begin{gather}
[\_, M_{s}^l] = \text{TransformerLayer}(\text{AssociativeLayer}([H_{s-1}^{l-1}, M_{s}^{l-1}]))
\\
k_{i},v_{i} =W_K m_{i},W_V m_{i}; \quad \beta_{i} = \sigma(W_\beta m_i); \quad
A_0^l = \vec{0}; \quad  z_0^l = \vec{0};\\
\overline{v}_i = \frac{A_{s-1}^l \phi(k_i)}{(z_{s-1})^T \phi(k_i)}; \quad \gamma_i = 1 - \frac{(z_{s-1})^T\phi(k_i)}{\|\phi(k_i)\|^2}; \\
A_s^l = A_{s-1}^l + \sum_i \beta_i (v_i - \overline{v}_i) \otimes \phi(k_i); \quad z^l_{s} = z^l_{s-1} + \sum_i\gamma_i \phi (k_i).\\
\text{AssociativeLayer}(x_i) = \frac{A_{s-1}^l \phi(W_Q x_i)}{(z^l_{s-1})^T \phi(W_Q x_i)},
\end{gather}
where $m_i$ is the vector from $M_s^l$, $A_s^l \in \mathbb{R}^{d_{\text{model}} \times 6d_{\text{mem}}}$, $z_{s}^l \in \mathbb{R}^{6d_{\text{mem}}}$, $\phi$ is the untrained nonlinearity DPFP-3 \citep{schlag2021lineartransformerssecretlyfast}, $x_i$. is the vector from $[H_{s-1}^{l-1}, M_{s}^{l-1}]$.

This mechanism in fact implements quasi-linear attention with delta-rule for segment-level recurrence.

\begin{figure}[t]
    \centering
      \includegraphics[width=\linewidth]{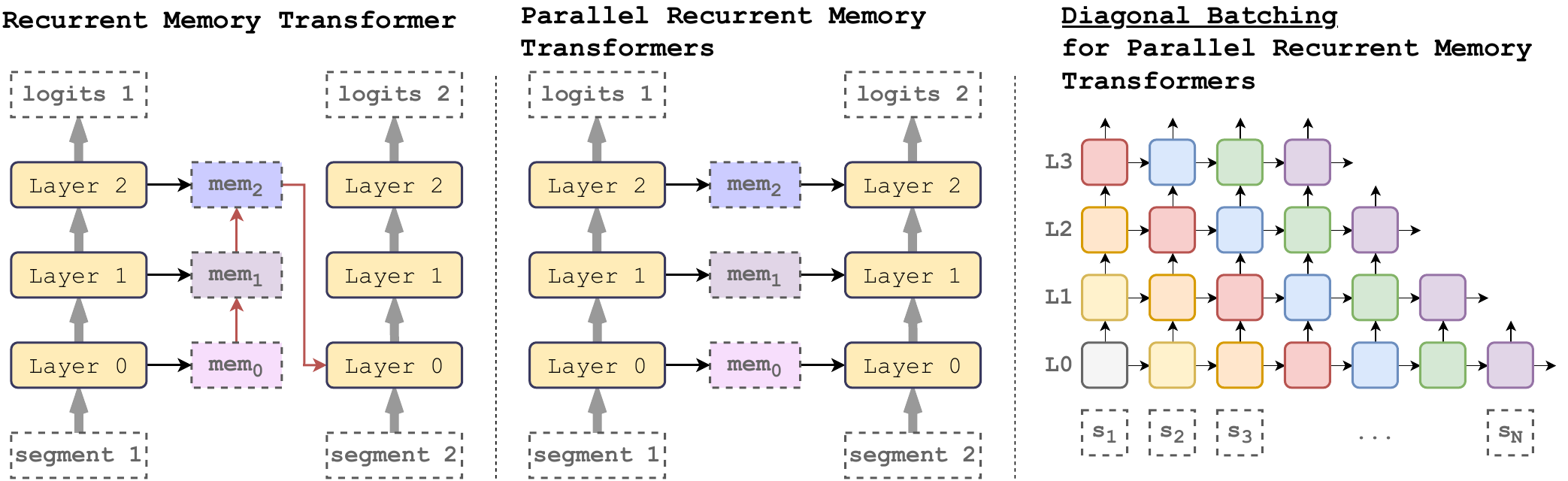}

    \caption{\textbf{Unlocking Parallelism in Recurrent Memory Transformers (RMT) with Diagonal Batching.} \textbf{Left:} Standard RMT splits long sequences and processes segments sequentially. Each layer updates a memory state ($\mathrm{mem}_0$, $\mathrm{mem}_1$, $\dots$) and the final-layer memory state is fed as input to the next segment; red arrows highlight the recurrent dependencies that force strictly sequential execution. \textbf{Center:} Parallel RMT generalizes a family of models with \emph{layer-level memory}. Each layer maintains its own memory state and passes it horizontally to the same layer in the next segment. This eliminates inter-layer memory flow, yet still requires processing segments in order within each layer, thereby creating layer-wise recurrence. \textbf{Right:} Diagonal Batching rearranges the 2D grid of layers (rows) and segments (columns) into independent "diagonals" (same colored blocks). This allows all operations on one diagonal (up to N\_Layers) to execute concurrently on the GPU, thus eliminating the sequential bottleneck while preserving all layer-level recurrence.}
    \label{fig:diagonal_visualization}
  \end{figure}

\subsection{Layer-level Recurrent Models}

Our method is primarily applicable to layer-level recurrent architectures, wherein the output of each segment (timestep) depends solely on the input and output of the preceding segment (timestep) within the same layer. We broadly refer to models that satisfy this assumption as Parallel Recurrent Memory Transformers (PRMTs, ~\Cref{fig:diagonal_visualization}, center): Associative Recurrent Memory Transformer (ARMT)~\citep{rodkin2024associative}, RWKV~\citep{peng-etal-2023-rwkv}, Mamba~\citep{mamba,dao2024mamba2}, and other linear-recurrent models~\citep{yangparallelizing}.

In ARMT, each layer $l$ has its own memory state that consists of associative matrix $A^{l}$. Memory state is updated by special associative block that takes as input outputs of the transformer layer $H_{t-1}^{l}$ on previous segment $t-1$ and memory update is defined as $A^{l}_{t} = \operatorname{AssociativeBlock}(A^{l}_{t-1}, H_{t-1}^{l})$. Inside the Associative Block, $A^{l}_{t}$ is updated by delta rule, in a simplified form: $A^{l}_{t} = A^{l}_{t-1} + v^l_t \otimes k^l_t$, where $v^l_t$ and $k^l_t$ are obtained by linear transformations of $H_{t-1}^{l}$. Each memory update in each layer is made once per segment.

This per-layer memory allows us to optimize the scheduling of which segments can be computed in parallel and at which layers.

There also exists a class of models that do not satisfy these assumptions. For instance, in RMT~\citep{bulatov2022recurrent}, the output of a given layer at segment $t$ additionally depends on the output of the final layer from the previous segment (\autoref{fig:diagonal_visualization}, left).

\subsection{Existing inference optimizations techniques for transformer models}

Several techniques are proposed to speed up the inference of transformer models, such as FlashAttention~\cite{dao2022flashattention,dao2023flashattention2}, speculative decoding~\cite{xia-etal-2023-speculative}, quantization techniques~\cite{frantar-gptq,lin2024awq}, and many others.

Therefore, any new approach should be compatible with these optimizations to be useful in practice. Diagonal Batching is independent of these methods and integrates with them seamlessly. It employs FlashAttention to group segments and achieve highly efficient attention computation.

\subsection{Hardware utilization}

Effectiveness of individual operations often analyzed via the roofline model, which characterizes the performance limits of hardware based on computational intensity and memory bandwidth \cite{williams2009roofline}. Transformer architecture mostly consists of matrix multiplication - compute bound operation. Matrix multiplication's computational intensity don't depends on batch size. However, the total achievable floating-point operations per second (FLOPS) improves significantly, as larger batch sizes enable better parallel workload distribution across GPU cores, optimizing hardware utilization \cite{dao2022flashattention}.

Despite these benefits, large batch size introduces significant memory demand. It mostly comes from intermediate activations computations and storing output logits, which scales linearly with batch size and sequence length. This limits practical usage of batching, as large language transformers often use almost all available GPU memory.

\section{Diagonal Batching method}
\label{sec:diag_method}

\subsection{Intuition and dependency graph}

In the naive approach, we must perform many forward operations 
(\texttt{n\_segments} $\times$ \texttt{n\_layers}) using inputs of shape 
(\texttt{segment\_size}, \texttt{hidden\_size}).

Due to parallel memory usage, each \texttt{(segment, layer)} pair only depends on 
the preceding pairs: \texttt{(segment, layer-1)} and \texttt{(segment-1, layer)}.

Given this dependency, all pairs where \texttt{segment + layer = i} can be computed 
in parallel during the $i$-th iteration. Each iteration can be visualized as a diagonal in the forward-pass computation graph, as shown in \Cref{fig:diagonal_visualization}, right.

If the execution is not compute-bound, this diagonal execution approach can yield significant speedup.
Note that this property holds only for parallel memory models. In recursive memory 
models, each \texttt{(segment, layer)} depends on all previous 
\texttt{(segment-k, layer-n)} pairs, making diagonal batching not applicable.

\subsection{Batching}

\begin{figure}[t]
    \centering
    \begin{subfigure}[c]{0.3\textwidth}
      \centering
      \includegraphics[width=\textwidth]{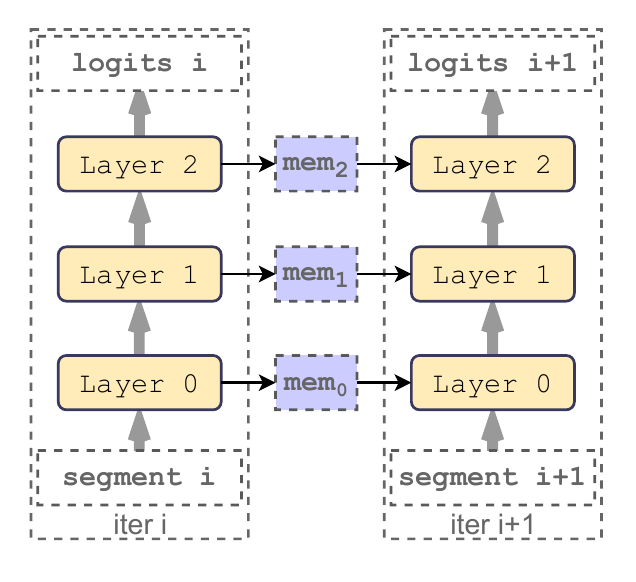}
      \caption{Baseline compute scheme.}
    \end{subfigure}%
    \begin{subfigure}[c]{0.7\textwidth}
        \centering
        \includegraphics[width=\textwidth]{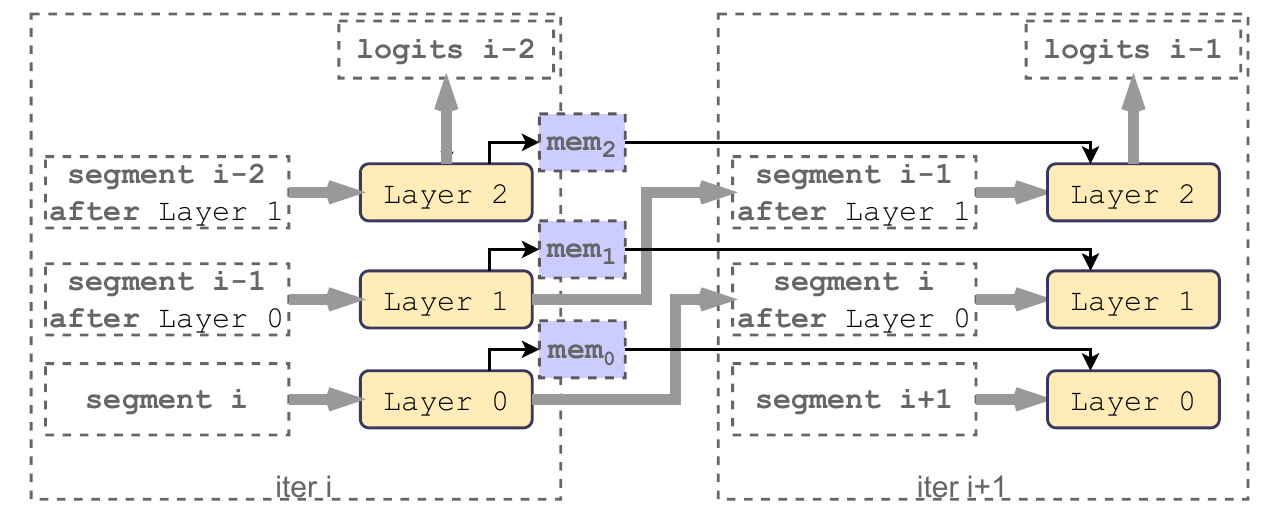}
        \caption{Diagonal Batching: grouped compute scheme.}
    \end{subfigure}%

    \caption{Baseline compute schedule in PRMTs leads to n\_layers x n\_segments sequential operations. Diagonal Batching reduces this value to n\_layers + n\_segments by grouped computations.}
  \label{fig:scheme_sequential_vs_grouped}
  \end{figure}

Simplified description of the algorithm is given for ARMT in Algorithm~\ref{alg:fastgroupedarmt}. 
For parallel RMT, the algorithm is the same, but without memory association and update.

\begin{algorithm}[t]
    \caption{\textsc{Grouped ARMT Execution}}
    \begin{algorithmic}[1]
    \Require input sequence $\mathcal{I}$, number of layers $L$, grouped layer $\mathcal{G}$

    \State \Call{ZeroGroupedMemory}{$\mathcal{M}$}
    \State $\textit{segments} \gets$ \Call{Segment}{$\mathcal{G}$, $\mathcal{I}$} \Comment{token ids to segments with memory tokens}

    \State $\textit{GInput} \gets []$,\; $\textit{Out} \gets []$
    \For{$i = 0$ \textbf{to} $\,L + |\textit{segments}| - 1$}
        \If{$i < |\textit{segments}|$}               
            \State \textbf{prepend} $\textit{segments}[i]$ to $\textit{GInput}$        \Comment{ingest new segment}
        \EndIf

        \State $X \gets \Call{Stack}{\textit{GInput}}$

        \If{$i>0$}
            \State $X_{0{:}|X|-1} \gets$ \Call{Associate}{$\mathcal{G},X_{0{:}|X|-1}$} \Comment{memory association operation between consecutive segments}
        \EndIf

        \State $Y \gets \Call{GroupedForward}{\mathcal{G},X}$ \Comment{multi-layer grouped call}
        \State \Call{UpdateMem}{$\mathcal{G}, Y_{:,-num\_mem\_tokens:}$} \Comment{memory update for next segment}
        \State $\textit{GInput} \gets$ list of segments in $Y$

        \If{$i \ge L-1$}                              
            \State $O \gets \textit{GInput}.\Call{PopLast}{}$ \Comment{segment went through all layers}
            \State \textbf{append} $O$ to $\textit{Out}$
        \EndIf
    \EndFor

    \State \Return $\Call{Concat}{\textit{Out}}$ \Comment{final logits}
    \end{algorithmic}
    \label{alg:fastgroupedarmt}
    \end{algorithm}


\begin{lemma}\label{lem:diagonal_grouping_thr}
    Diagonal Batching completes the DAG in the minimum possible number of
    groups, \(N_{\mathrm{segments}}+N_{\mathrm{layers}}-1\), and schedules each node
    \((i,j)\) in its earliest feasible group \(i+j\).
\end{lemma}
\begin{proof}
    Topologically sort the DAG by the key \((i,j)\) with root being \((0,0)\).  In this ordering, each node
    \((i,j)\) appears at level \(i+j\), which is therefore the earliest
    group it can occupy, and the longest path has length
    \(N_{\mathrm{segment}}+N_{\mathrm{layers}}-1\) vertices.  Hence, any schedule needs at
    least \(N_{\mathrm{segment}}+N_{\mathrm{layers}}-1\) groups.  Diagonal batching
    uses precisely those levels as its groups, achieving both bounds.
\end{proof}

\subsection{Implementation details}

To efficiently implement grouped layer computations, we modify the base model architecture. All layers are replaced with a single grouped layer, as shown in Figure~{\ref{fig:scheme_sequential_vs_grouped}}. Using the initial layer of the model as the basis, we implement the following adjustments:
\begin{enumerate}
\item Replace the linear layers with a \texttt{GroupedMatmul} operation. The weights and biases are constructed by stacking those from the original layers.
\item Layer normalization weights are also replaced by stacking parameters across all layers. Additionally, the forward pass is adapted to ensure correct broadcasting behavior.
\item All other operations remain unchanged. However, they operate as though they handle significantly larger batch sizes, contributing to parallel execution.
\end{enumerate}

For the grouped matrix multiplication, we utilize the \texttt{GroupedGEMM} function from the CUTLASS library with a minor optimization: the output tensor is pre-allocated as a single large tensor, which is subsequently partitioned into individual submatrices without additional overhead.



\section{Experiments}
\label{sec:experiments}

In experiment section, we address two main questions regarding diagonal batching method: 
\begin{itemize}
    \item How much speedup we can get compared to naive ARMT setup in single request inferences.
    \item How the proposed method compares with batching strategies.
\end{itemize}

We start from showing efficiency grows for individual bottleneck operations inside network - linear layers and attention. Then we show the resulting scaling for the transformer models with ARMT of different sizes. We conducted all experiments with the models from the Llama-3 family~\cite{grattafiori2024llama}.

\subsection{Linear layer efficiency}

The only change from base model is that we substitute linear layer with matrix multiplication to layers with grouped GEMM with group equal to all linear layers weights. 
In Figure~{\ref{fig:gemm_vs_group_over_matmul_scaling_seq1024_hid2048_a100}}
we show, that grouped GEMM FLOPS scales similar throw group size to GEMM with corresponding batch size. This gives the basis that our method should scale similar to underlying model with batch size as all other operations basically the same (but in different order).

Second, we have group size equal to the number of layers in the model. This way, we move grouped GEMM operation to peak GEMM flops for a100 and h100 GPUs, ensuring high utilization. Corresponding FLOPS improvement shown in~\Cref{fig:gemm_vs_group_over_matmul_scaling_seq1024_hid2048_a100}.



\begin{figure}[h]
    \centering
    \begin{subfigure}[c]{0.5\textwidth}
      \centering
      \includegraphics[width=\textwidth]{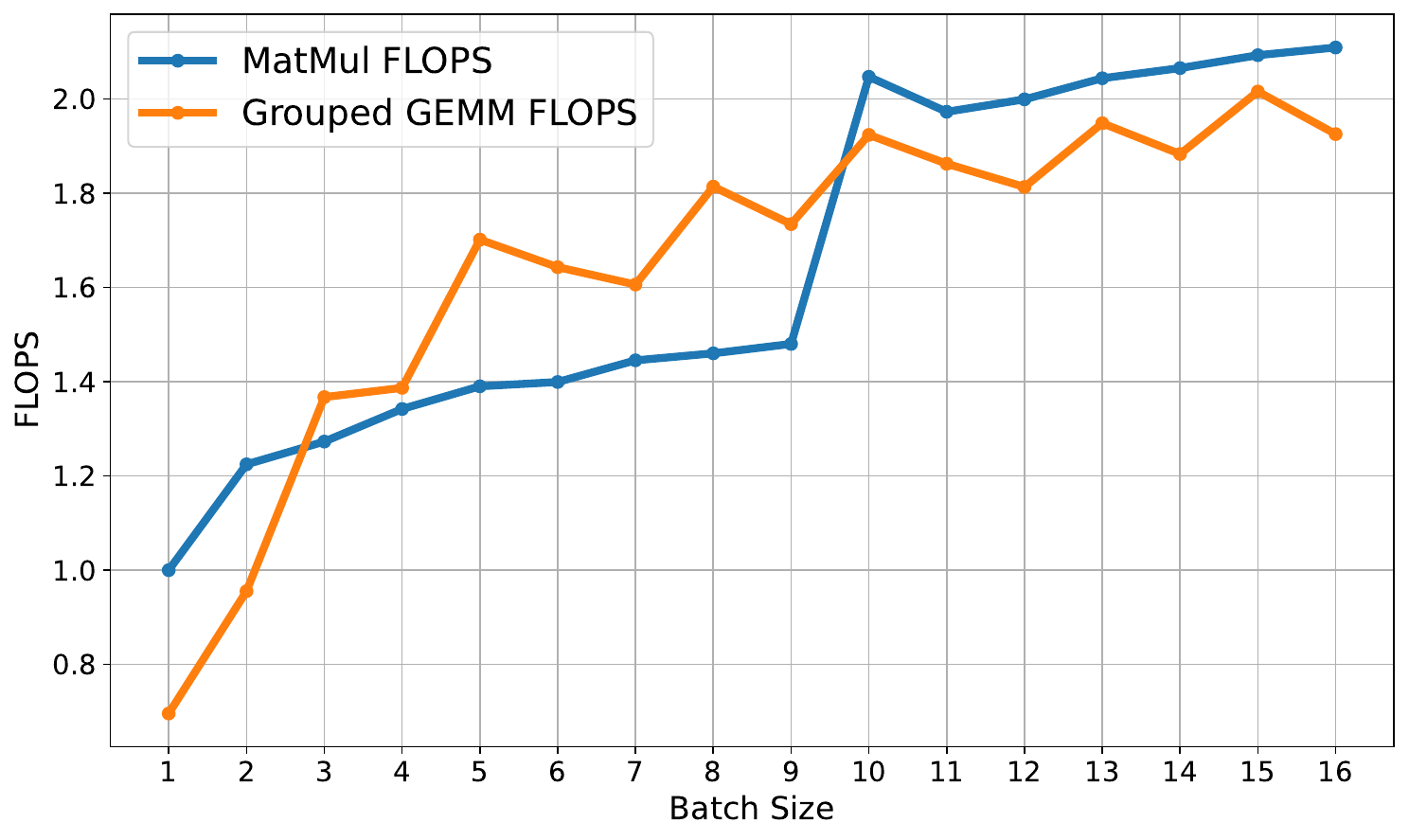}
      \caption{a100}
    \end{subfigure}%
    \centering
    \begin{subfigure}[c]{0.5\textwidth}
      \centering
      \includegraphics[width=\textwidth]{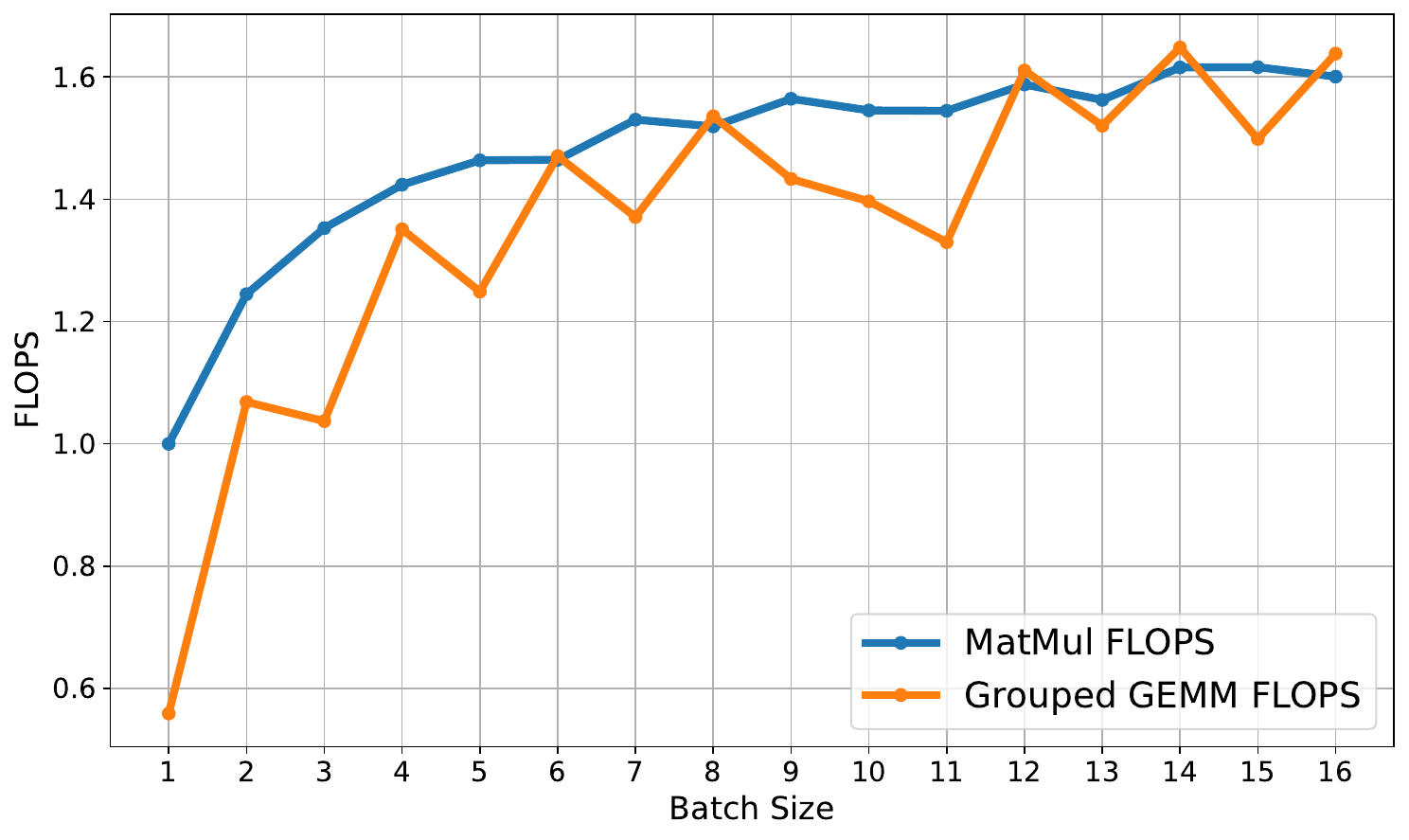}
      \caption{h100}
    \end{subfigure}%

    \caption{Cutlass Group GEMM scales similarly to batch size 1 Linear layer's matrix multiplication, starting from group size 4.}
    \label{fig:gemm_vs_group_over_matmul_scaling_seq1024_hid2048_a100}
\end{figure}

\subsection{Attention layer efficiency}

Our method does not modify attention layer at all. Instead, attention just performs batched operation with batch size equal to number of layers. This increase its performance to implementation FLOPS peak.
We show relative FLOPS speedups in \Cref{fig:attn_scaling_over_batch1_seq1024_hid2048_head128_a100}. 



\begin{figure}[h]
    \centering
    \begin{subfigure}[c]{0.45\textwidth}
      \centering
      \includegraphics[width=\textwidth]{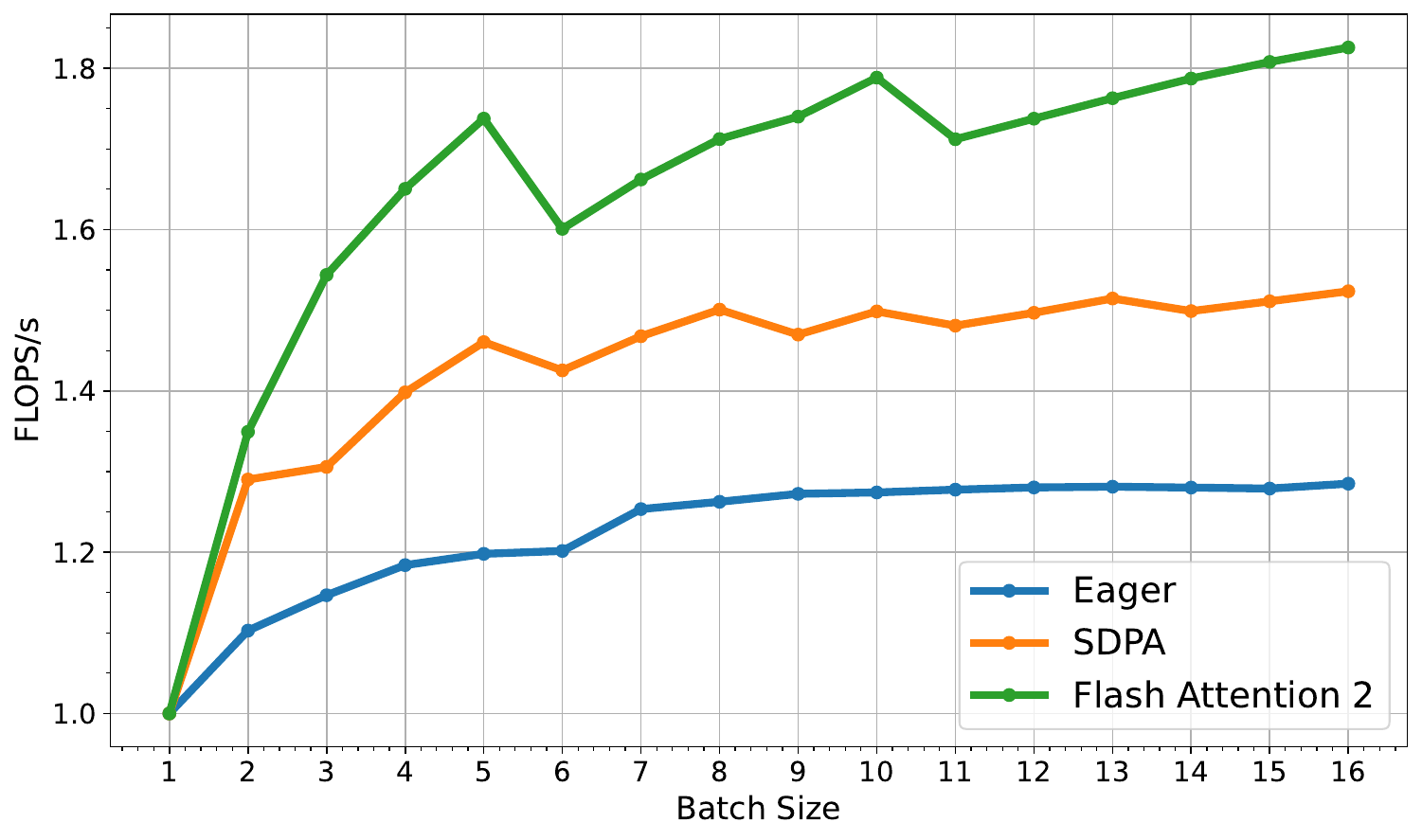}
      \caption{a100}
    \end{subfigure}%
    \centering
    \begin{subfigure}[c]{0.45\textwidth}
      \centering
      \includegraphics[width=\textwidth]{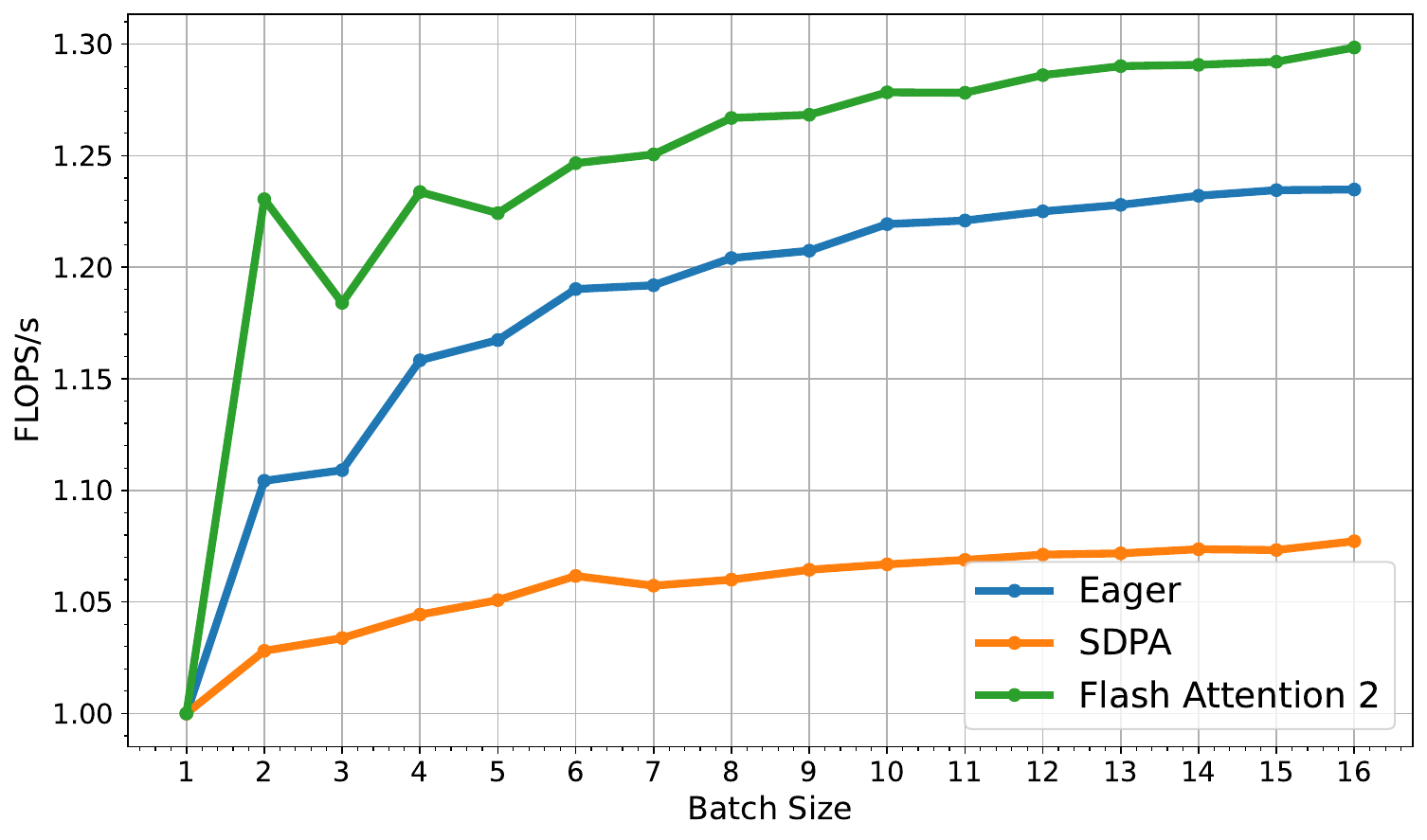}
      \caption{h100}
    \end{subfigure}%

    \caption{Diagonal batching increase attention performance by treating groups as batches—similar to increasing the model's overall batch size.}
    \label{fig:attn_scaling_over_batch1_seq1024_hid2048_head128_a100}
\end{figure}

\subsection{Inference scaling}

The performance increase for individual operations directly translates into overall model speedup. We evaluate this effect on Llama ARMT models of varying sizes—160M (\Cref{tab:perf_comparison_llama160m}), 1B (\Cref{tab:merged_perf_speedup_llama1b}), 3B (\Cref{tab:perf_comparison_llama3b}), and 8B (\Cref{tab:perf_comparison_llama8b}).

Across all model sizes and batch configurations, our implementation consistently achieves substantial speedups over the default ARMT implementation. Gains are particularly pronounced for smaller segment sizes. This is because, with larger matrix multiplications, hardware utilization is already near peak FLOPS, leaving less room for group scaling.

A key implication of these results is that researchers can prioritize quality-driven choices for segment size without being overly constrained by performance. Diagonal batching decouples performance from segment size, allowing better flexibility in architectural decisions.

\begin{table}[h]
  \centering
  \renewcommand{\arraystretch}{1.2}
  \resizebox{\textwidth}{!}{%
  \begin{tabular}{l*{6}{l}}
  \toprule
  \textbf{Method} & \multicolumn{6}{c}{\textbf{Sequence Length}} \\
  \cmidrule(lr){2-7}
   & {\textbf{4096}} & {\textbf{8192}} & {\textbf{16384}} & {\textbf{32768}} & {\textbf{65536}} & {\textbf{131072}} \\
  \midrule
  Llama-3.2-1B & 0.024 & 0.026 & 0.376 & 0.926 & 2.460 & 8.160 \\
  \rowcolor{gray!10} \textbf{Configuration: (512, 128)} \\
  LLama-3.2-1B-ARMT & 0.147 & 0.574 & 1.15 & 2.29 & 4.52 & 8.98 \\
  Diagonal Batching: LLama-3.2-1B-ARMT & 0.283 \textcolor{darkred}{x0.52} & 0.248 \textcolor{teal}{x2.32} & 0.454 \textcolor{teal}{x2.53} & 0.861 \textcolor{teal}{x2.66} & 1.67 \textcolor{teal}{x2.71} & 3.3 \textcolor{teal}{x2.72} \\
  \midrule
  \rowcolor{gray!10} \textbf{Configuration: (1024, 128)} \\
  LLama-3.2-1B-ARMT & 0.149 & 0.291 & 0.578 & 1.15 & 2.3 & 4.48 \\
  Diagonal Batching: LLama-3.2-1B-ARMT & 0.119 \textcolor{teal}{x1.25} & 0.196 \textcolor{teal}{x1.49} & 0.351 \textcolor{teal}{x1.65} & 0.656 \textcolor{teal}{x1.75} & 1.27 \textcolor{teal}{x1.81} & 2.48 \textcolor{teal}{x1.81} \\
  \midrule
  \rowcolor{gray!10} \textbf{Configuration: (2048, 128)} \\
  LLama-3.2-1B-ARMT & 0.094 & 0.177 & 0.344 & 0.679 & 1.35 & 2.68 \\
  Diagonal Batching: LLama-3.2-1B-ARMT & 0.108 \textcolor{darkred}{x0.87} & 0.176 \textcolor{teal}{x1.01} & 0.304 \textcolor{teal}{x1.13} & 0.571 \textcolor{teal}{x1.19} & 1.11 \textcolor{teal}{x1.22} & 2.18 \textcolor{teal}{x1.23} \\
  \midrule
  \rowcolor{gray!10} \textbf{Configuration: (4096, 128)} \\
  LLama-3.2-1B-ARMT & 0.082 & 0.155 & 0.301 & 0.594 & 1.18 & 2.35 \\
  Diagonal Batching: LLama-3.2-1B-ARMT & 0.102 \textcolor{darkred}{x0.80} & 0.172 \textcolor{darkred}{x0.90} & 0.295 \textcolor{teal}{x1.02} & 0.553 \textcolor{teal}{x1.07} & 1.07 \textcolor{teal}{x1.10} & 2.1 \textcolor{teal}{x1.12} \\
  \bottomrule
  \end{tabular}%
  }
  \caption{Diagonal Batching allows to speed-up the execution for longer sequences — from 1.1× to 2.7× compared to base ARMT at 131072 sequence length. Execution time comparison (in seconds) and relative speedups across different sequence lengths compared to LLama-3.2-1B-ARMT. Configuration format: (segment\_size, memory\_tokens). Measured on Nvidia A100 GPU.}
  \label{tab:merged_perf_speedup_llama1b}
\end{table}

\subsection{Diagonal batching vs mini-batching}

We evaluate the effectiveness of diagonal batching compared to standard mini-batching by measuring compute time per segment under identical hardware and model configurations. As shown in Figure~\ref{fig:batch_vs_group_llamas}, diagonal batching achieves compute scaling per segment that closely matches micro-batching across almost all tested scenarios.

To provide an upper bound on achievable performance, we also report the Ideal Even Load case, than all segments computations computed with full grouped layer with maximum achievable FLOPS. One can see this even load setup is much better, mostly matching or overcoming the biggest batch sizes. The gap between them is our current implementation inefficiency.

Notably, diagonal batching delivers substantial performance improvements for larger models (starting from 1B parameters), particularly when segment sizes are moderate. For these configurations, diagonal batching matches large batch sizes.

These findings suggest that diagonal batching effectively captures the utilization benefits of large-batch inference—through parallelized scheduling rather than increased memory allocation.

\begin{figure}[H]
        \centering
        \includegraphics[width=\textwidth]{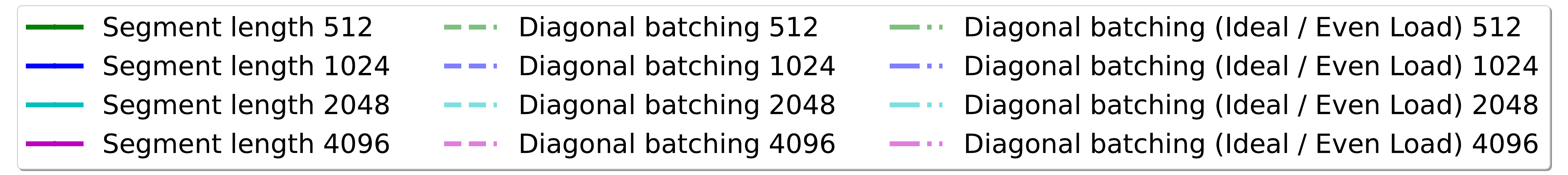}

        \centering
        \includegraphics[width=\textwidth]{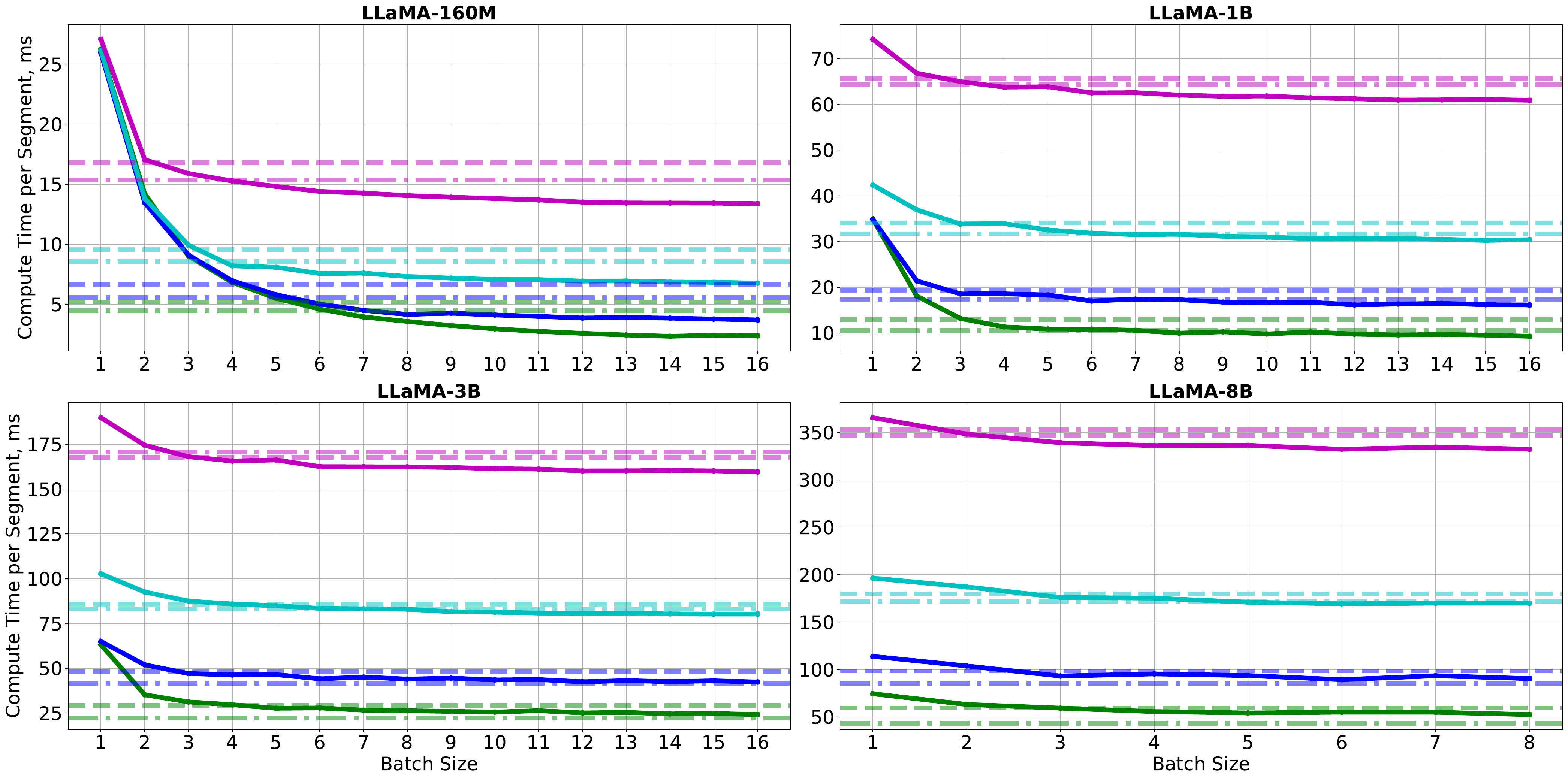}
    \caption{Ideal batch-size scaling vs grouped batching on Nvidia A100 for Llama models, time per segment in batch (group)}
    \label{fig:batch_vs_group_llamas}
\end{figure}

\subsection{Error accumulation}
We conducted an empirical investigation on error accumulation during inference stage with Diagonal Batching. Our experiments show that the overall error is less than 2\% for all sequences shorter than 32,768 tokens. This is comparable to other efficient layers implementations used in production. For example, we observed FlashAttention2~\citep{dao2023flashattention2} gives 1-2\% relative logits error compared to other attention implementations on same random input sequences. 

The detailed error values for each segment are presented in \Cref{tab:error_accumulation}. The error is calculated as the ratio of the Frobenius norm of difference between logits of base ARMT implementation and logits of ARMT with Diagonal Batching to the norm of logits of base ARMT. However, we the effect of error accumulation on downstream tasks is negligible. To prove this, we evaluated the trained ARMT model both in original implementation and with Diagonal Batching; the results are presented in~\Cref{tab:babilong_scores} in~Appendix~\ref{app:training_diag}. These results show that both implementations achieve the same results on the BABILong benchmark~\cite{babilong}, while \Cref{tab:babilong_time} in Appendix~\ref{app:training_diag} shows that diagonal batching can increase the relative speed by up to 3.2x for 64k-length token sequences.


\begin{table}[h]
\centering
\begin{tabular}{c||c|c|c|c|c|c}
\toprule
Number of segments & 1 & 2 & 4 & 8 & 16 & 32 \\
\midrule
Diagonal Batching, Error, \% & 0.00  & 1.10  & 1.49 & 1.75 & 1.89 & 1.87 \\
\midrule
FlashAttention2~\citep{dao2023flashattention2} vs torch SDPA, Error, \% & 1.25  & 1.15 & 1.17 & 1.22 & 1.36 & 1.45 \\
\bottomrule
\end{tabular}
\caption{\label{tab:error_accumulation} During inference with diagonal batching, error accumulates but does not exceed 2\%, which is comparable to the change of attention implementation (FlashAttention vs SDPA). The results for ARMT with Llama-3.2-1B-Instruct are shown with a segment size of 1024 tokens.}
\end{table}

\section{Conclusion}
\label{sec:conclusion}

Long-context inference with transformer models still suffers from quadratic compute and linear memory growth. 
Several linear complexity architectures, such as Mamba, RWKV, and Recurrent Memory Transformers (RMTs), aim to address this. 
RMTs, in particular, offer the advantage of minimal architectural changes, ensuring compatibility with existing models and algorithms.

This paper demonstrated that the principal bottleneck in both RMTs and their layer-memory variants (PRMTs) is not algorithmic complexity but scheduling: recurrent dependencies force fine-grained synchronization, that underutilizes modern accelerators. We introduced \emph{Diagonal Batching}, a simple but powerful scheduling scheme that reorganizes the layer--segment computation grid into concurrency-friendly diagonals, thereby enabling up to N\_Layers operations per kernel without altering exact recurrence. Our experiments demonstrate that a Llama-1B ARMT equipped with diagonal batching achieves 
a 3.3x latency decrease over the vanilla Llama-1B and a 1.8x speedup 
over a sequential RMT implementation on a 131,072 token context task, 
all while maintaining high exactness of resulting logits (with only a 1\% relative error).


Considering these advantages, Diagonal Batching turns theoretically appealing compute scaling of PRMTs into a practical solution for exact linear-time inference on extremely long contexts. By eliminating the major performance barrier, it positions memory-augmented recurrent Transformers as a competitive and scalable foundation for next-generation LLM applications that require efficient long-range input processing.

\section*{Limitations}
\label{sec:limitations}
Despite its advantages, Diagonal Batching has several practical limitations. First, it is not directly compatible with the Recurrent Memory Transformers (RMTs) due to intra-layer recurrence. However, a more promising approach is to focus on Parallel RMTs, which has already been shown in previous works to be more effective~\citep{rodkin2024associative}. Second, our current implementation assumes a uniform layer configuration. When models employ heterogeneous layers or varied hidden sizes, applying the technique requires more intricate grouping logic and manual engineering. Finally, the achievable speedup increases with the number of layers. Therefore, shallower models or models with very few layers will only see modest performance gains.


\bibliographystyle{plain}
\bibliography{bibliography}

\begin{thebibliography}{10}

\bibitem{ainslie2023gqa}
Joshua Ainslie, James Lee-Thorp, Michiel de~Jong, Yury Zemlyanskiy, Federico
  Lebron, and Sumit Sanghai.
\newblock Gqa: Training generalized multi-query transformer models from
  multi-head checkpoints.
\newblock In {\em Proceedings of the 2023 Conference on Empirical Methods in
  Natural Language Processing}, pages 4895--4901, 2023.

\bibitem{bulatov2023scaling}
Aydar Bulatov, Yuri Kuratov, Yermek Kapushev, and Mikhail Burtsev.
\newblock Beyond attention: Breaking the limits of transformer context length
  with recurrent memory.
\newblock {\em Proceedings of the AAAI Conference on Artificial Intelligence},
  38(16):17700--17708, Mar. 2024.

\bibitem{bulatov2022recurrent}
Aydar Bulatov, Yury Kuratov, and Mikhail Burtsev.
\newblock Recurrent memory transformer.
\newblock {\em Advances in Neural Information Processing Systems},
  35:11079--11091, 2022.

\bibitem{dai2019transformerxl}
Zihang Dai, Zhilin Yang, Yiming Yang, Jaime Carbonell, Quoc Le, and Ruslan
  Salakhutdinov.
\newblock Transformer-{XL}: Attentive language models beyond a fixed-length
  context.
\newblock In {\em Proceedings of the 57th Annual Meeting of the Association for
  Computational Linguistics}, pages 2978--2988, Florence, Italy, July 2019.
  Association for Computational Linguistics.

\bibitem{dao2023flashattention2}
Tri Dao.
\newblock Flash{A}ttention-2: Faster attention with better parallelism and work
  partitioning.
\newblock In {\em International Conference on Learning Representations (ICLR)},
  2024.

\bibitem{dao2022flashattention}
Tri Dao, Daniel~Y. Fu, Stefano Ermon, Atri Rudra, and Christopher R{\'e}.
\newblock Flash{A}ttention: Fast and memory-efficient exact attention with
  {IO}-awareness.
\newblock In {\em Advances in Neural Information Processing Systems (NeurIPS)},
  2022.

\bibitem{dao2024mamba2}
Tri Dao and Albert Gu.
\newblock Transformers are ssms: Generalized models and efficient algorithms
  through structured state space duality.
\newblock In {\em International Conference on Machine Learning}, pages
  10041--10071. PMLR, 2024.

\bibitem{devlin2019bert}
Jacob Devlin, Ming-Wei Chang, Kenton Lee, and Kristina Toutanova.
\newblock {{BERT}}: {{Pre}}-training of {{Deep Bidirectional Transformers}} for
  {{Language Understanding}}.
\newblock In {\em Proceedings of the 2019 {{Conference}} of the {{North
  American Chapter}} of the {{Association}} for {{Computational Linguistics}}:
  {{Human Language Technologies}}, {{Volume}} 1 ({{Long}} and {{Short
  Papers}})}, pages 4171--4186, 2019.

\bibitem{dubey2024llama3}
Abhimanyu Dubey, Abhinav Jauhri, Abhinav Pandey, Abhishek Kadian, Ahmad
  Al-Dahle, Aiesha Letman, Akhil Mathur, Alan Schelten, Amy Yang, Angela Fan,
  et~al.
\newblock The llama 3 herd of models.
\newblock {\em arXiv preprint arXiv:2407.21783}, 2024.

\bibitem{frantar-gptq}
Elias Frantar, Saleh Ashkboos, Torsten Hoefler, and Dan Alistarh.
\newblock {GPTQ}: Accurate post-training compression for generative pretrained
  transformers.
\newblock {\em arXiv preprint arXiv:2210.17323}, 2022.

\bibitem{grattafiori2024llama}
Aaron Grattafiori, Abhimanyu Dubey, Abhinav Jauhri, Abhinav Pandey, Abhishek
  Kadian, Ahmad Al-Dahle, Aiesha Letman, Akhil Mathur, Alan Schelten, Alex
  Vaughan, et~al.
\newblock The llama 3 herd of models.
\newblock {\em arXiv preprint arXiv:2407.21783}, 2024.

\bibitem{mamba}
Albert Gu and Tri Dao.
\newblock Mamba: Linear-time sequence modeling with selective state spaces.
\newblock {\em arXiv preprint arXiv:2312.00752}, 2023.

\bibitem{gu2021s4}
Albert Gu, Karan Goel, and Christopher Re.
\newblock Efficiently modeling long sequences with structured state spaces.
\newblock In {\em International Conference on Learning Representations}, 2021.

\bibitem{hutchins2022blockrecurrent}
DeLesley Hutchins, Imanol Schlag, Yuhuai Wu, Ethan Dyer, and Behnam Neyshabur.
\newblock Block-recurrent transformers.
\newblock In Alice~H. Oh, Alekh Agarwal, Danielle Belgrave, and Kyunghyun Cho,
  editors, {\em Advances in Neural Information Processing Systems}, 2022.

\bibitem{jacobs2023deepspeedulysses}
Sam~Ade Jacobs, Masahiro Tanaka, Chengming Zhang, Minjia Zhang, Shuaiwen~Leon
  Song, Samyam Rajbhandari, and Yuxiong He.
\newblock Deepspeed ulysses: System optimizations for enabling training of
  extreme long sequence transformer models.
\newblock {\em arXiv preprint arXiv:2309.14509}, 2023.

\bibitem{jelassi2024repeat}
Samy Jelassi, David Brandfonbrener, Sham~M Kakade, and Eran Malach.
\newblock Repeat after me: Transformers are better than state space models at
  copying.
\newblock In {\em International Conference on Machine Learning}, pages
  21502--21521. PMLR, 2024.

\bibitem{babilong}
Yuri Kuratov, Aydar Bulatov, Petr Anokhin, Ivan Rodkin, Dmitry Sorokin, Artyom
  Sorokin, and Mikhail Burtsev.
\newblock Babilong: Testing the limits of llms with long context
  reasoning-in-a-haystack.
\newblock In A.~Globerson, L.~Mackey, D.~Belgrave, A.~Fan, U.~Paquet,
  J.~Tomczak, and C.~Zhang, editors, {\em Advances in Neural Information
  Processing Systems}, volume~37, pages 106519--106554. Curran Associates,
  Inc., 2024.

\bibitem{xFormers2022}
Benjamin Lefaudeux, Francisco Massa, Diana Liskovich, Wenhan Xiong, Vittorio
  Caggiano, Sean Naren, Min Xu, Jieru Hu, Marta Tintore, Susan Zhang, Patrick
  Labatut, Daniel Haziza, Luca Wehrstedt, Jeremy Reizenstein, and Grigory
  Sizov.
\newblock xformers: A modular and hackable transformer modelling library.
\newblock \url{https://github.com/facebookresearch/xformers}, 2022.

\bibitem{lin2024awq}
Ji~Lin, Jiaming Tang, Haotian Tang, Shang Yang, Wei-Ming Chen, Wei-Chen Wang,
  Guangxuan Xiao, Xingyu Dang, Chuang Gan, and Song Han.
\newblock Awq: Activation-aware weight quantization for on-device llm
  compression and acceleration.
\newblock {\em Proceedings of Machine Learning and Systems}, 6:87--100, 2024.

\bibitem{liu2024deepseekv2}
Aixin Liu, Bei Feng, Bin Wang, Bingxuan Wang, Bo~Liu, Chenggang Zhao, Chengqi
  Dengr, Chong Ruan, Damai Dai, Daya Guo, et~al.
\newblock Deepseek-v2: A strong, economical, and efficient mixture-of-experts
  language model.
\newblock {\em arXiv preprint arXiv:2405.04434}, 2024.

\bibitem{ringattention}
Hao Liu, Matei Zaharia, and Pieter Abbeel.
\newblock Ringattention with blockwise transformers for near-infinite context.
\newblock In {\em The Twelfth International Conference on Learning
  Representations}, 2024.

\bibitem{merrill2024illusion}
William Merrill, Jackson Petty, and Ashish Sabharwal.
\newblock The illusion of state in state-space models.
\newblock In {\em International Conference on Machine Learning}, pages
  35492--35506. PMLR, 2024.

\bibitem{openai2023gpt4}
OpenAI.
\newblock Gpt-4 technical report, 2023.

\bibitem{peng-etal-2023-rwkv}
Bo~Peng, Eric Alcaide, Quentin Anthony, Alon Albalak, Samuel Arcadinho, Stella
  Biderman, Huanqi Cao, Xin Cheng, Michael Chung, Leon Derczynski, Xingjian Du,
  Matteo Grella, Kranthi Gv, Xuzheng He, Haowen Hou, Przemyslaw Kazienko, Jan
  Kocon, Jiaming Kong, Bart{\l}omiej Koptyra, Hayden Lau, Jiaju Lin, Krishna
  Sri~Ipsit Mantri, Ferdinand Mom, Atsushi Saito, Guangyu Song, Xiangru Tang,
  Johan Wind, Stanis{\l}aw Wo{\'z}niak, Zhenyuan Zhang, Qinghua Zhou, Jian Zhu,
  and Rui-Jie Zhu.
\newblock {RWKV}: Reinventing {RNN}s for the transformer era.
\newblock In Houda Bouamor, Juan Pino, and Kalika Bali, editors, {\em Findings
  of the Association for Computational Linguistics: EMNLP 2023}, pages
  14048--14077, Singapore, December 2023. Association for Computational
  Linguistics.

\bibitem{radford2019gpt2}
Alec Radford, Jeff Wu, Rewon Child, David Luan, Dario Amodei, and Ilya
  Sutskever.
\newblock Language models are unsupervised multitask learners.
\newblock 2019.

\bibitem{rae2019compressive}
Jack~W. Rae, Anna Potapenko, Siddhant~M. Jayakumar, Chloe Hillier, and
  Timothy~P. Lillicrap.
\newblock Compressive transformers for long-range sequence modelling.
\newblock In {\em International Conference on Learning Representations}, 2020.

\bibitem{reid2024gemini}
Machel Reid, Nikolay Savinov, Denis Teplyashin, Dmitry Lepikhin, Timothy
  Lillicrap, Jean-baptiste Alayrac, Radu Soricut, Angeliki Lazaridou, Orhan
  Firat, Julian Schrittwieser, et~al.
\newblock Gemini 1.5: Unlocking multimodal understanding across millions of
  tokens of context.
\newblock {\em arXiv preprint arXiv:2403.05530}, 2024.

\bibitem{rodkin2024associative}
Ivan Rodkin, Yuri Kuratov, Aydar Bulatov, and Mikhail Burtsev.
\newblock Associative recurrent memory transformer.
\newblock {\em CoRR}, 2024.

\bibitem{schlag2021lineartransformerssecretlyfast}
Imanol Schlag, Kazuki Irie, and Jürgen Schmidhuber.
\newblock Linear transformers are secretly fast weight programmers, 2021.

\bibitem{shazeer2019mqa}
Noam Shazeer.
\newblock Fast transformer decoding: One write-head is all you need.
\newblock {\em arXiv preprint arXiv:1911.02150}, 2019.

\bibitem{strobl2024formal}
Lena Strobl, William Merrill, Gail Weiss, David Chiang, and Dana Angluin.
\newblock What formal languages can transformers express? a survey.
\newblock {\em Transactions of the Association for Computational Linguistics},
  12, 2024.

\bibitem{sukhbaatar2015endtoend}
Sainbayar Sukhbaatar, Arthur Szlam, Jason Weston, and Rob Fergus.
\newblock End-to-end memory networks, 2015.

\bibitem{sun2023retentive}
Yutao Sun, Li~Dong, Shaohan Huang, Shuming Ma, Yuqing Xia, Jilong Xue, Jianyong
  Wang, and Furu Wei.
\newblock Retentive network: A successor to transformer for large language
  models.
\newblock {\em arXiv preprint arXiv:2307.08621}, 2023.

\bibitem{vaswani2017attention}
Ashish Vaswani, Noam Shazeer, Niki Parmar, Jakob Uszkoreit, Llion Jones,
  Aidan~N Gomez, {\L}ukasz Kaiser, and Illia Polosukhin.
\newblock {{Attention is All you Need}}.
\newblock In {\em Advances in neural information processing systems}, pages
  5998--6008, 2017.

\bibitem{weston2014memory}
Jason Weston, Sumit Chopra, and Antoine Bordes.
\newblock Memory networks.
\newblock In Yoshua Bengio and Yann LeCun, editors, {\em 3rd International
  Conference on Learning Representations, {ICLR} 2015, San Diego, CA, USA, May
  7-9, 2015, Conference Track Proceedings}, 2015.

\bibitem{williams2009roofline}
Samuel Williams, Andrew Waterman, and David Patterson.
\newblock Roofline: an insightful visual performance model for multicore
  architectures.
\newblock {\em Communications of the ACM}, 52(4):65--76, 2009.

\bibitem{xia-etal-2023-speculative}
Heming Xia, Tao Ge, Peiyi Wang, Si-Qing Chen, Furu Wei, and Zhifang Sui.
\newblock Speculative decoding: Exploiting speculative execution for
  accelerating seq2seq generation.
\newblock In Houda Bouamor, Juan Pino, and Kalika Bali, editors, {\em Findings
  of the Association for Computational Linguistics: EMNLP 2023}, pages
  3909--3925, Singapore, December 2023. Association for Computational
  Linguistics.

\bibitem{yangparallelizing}
Songlin Yang, Bailin Wang, Yu~Zhang, Yikang Shen, and Yoon Kim.
\newblock Parallelizing linear transformers with the delta rule over sequence
  length.
\newblock In {\em The Thirty-eighth Annual Conference on Neural Information
  Processing Systems}.

\end{thebibliography}

%
%
%
%
%
%
%

\newpage
\appendix



\section{Evaluating Models with Diagonal Batching}
\label{app:training_diag}

Although diagonal Batching significantly speeds up the inference, it also introduces some numerical drifts due to the optimized execution procedure. To estimate the effect of these drifts on practical tasks, we evaluated the ARMT model on BABILong benchmark~\cite{babilong} with and without diagonal Batching. The ARMT model was trained on the BABILong dataset with curriculum learning on length up to 8192 tokens, similar to the approach described in~\cite{babilong}. After, we evaluated this model with and without diagonal batching on QA1 and QA2 tasks from BABILong. Note that we did not change the weights of the model in this experiment; we simply applied the proposed Diagonal Batching grouping method.

The evaluation results are presented in~\Cref{tab:babilong_scores}. As one can see, despite the numerical drifts during forward pass, the generation results remain almost unchanged up to the 65536 input length. These results show that diagonal batching preserves the quality of the generation of trained ARMT model and can be used as drop-in replacement to speed-up the inference.
\begin{table}
\centering
  \resizebox{0.5\textwidth}{!}{%
    \begin{tabular}{l|l|c|c}
    \toprule
     \textbf{Task} & \begin{tabular}[c]{@{}c@{}}\textbf{Length,}\\ \textbf{tokens} \end{tabular} & \begin{tabular}[c]{@{}c@{}}\textbf{LLama-3.2-1B}\\ \textbf{ARMT}\\ \end{tabular}  & \begin{tabular}[c]{@{}c@{}}\textbf{LLama-3.2-1B}\\ \textbf{ARMT,}\\ \textbf{Diagonal Batching}\\ \end{tabular}\\
     \midrule
    \multirow[c]{8}{*}{QA1} & 0K & 100 & 100 \\
    &  1K & 100 & 100 \\
    &  2K & 100 & 100 \\
    &  4K & 100 & 100 \\
    &  8K & 100 & 100 \\
    &  16K & 100 & 100 \\
    &  32K & 100 & 100 \\
    &  64K & 70 & 69 \\
    \midrule
    \multirow[c]{8}{*}{QA2} & 0K & 100 & 100 \\
    & 1K & 100 & 100 \\
    & 2K & 100 & 100 \\
    & 4K & 100 & 100 \\
    & 8K & 99 & 100 \\
    & 16K & 98 & 98 \\
    & 32K & 94 & 94 \\
    & 64K & 47 & 46 \\
    \bottomrule
    \end{tabular}%
  }
  \caption{Diagonal Batching maintains the same scores as the original ARMT inference method on the BABILong benchmark. Scores of the models were evaluated on the first two tasks: QA1 and QA2.}
  \label{tab:babilong_scores}
\end{table}

We also compared the inference time of these two approaches on the same benchmark. In this experiment, we measure not the forward pass time, but the generation time on the BABILong. \Cref{tab:babilong_time} shows that the diagonal batching approach significantly speeds up the generation, up to 3 times on the input length of 65536 tokens. During both of these experiments, we used the following ARMT configuration - the size of the segment was set to 1024 tokens, the number of memory tokens was set to 16 and the associative memory hidden size is 64.
\begin{table}
\centering
  \resizebox{0.6\textwidth}{!}{%
    \begin{tabular}{l|l|c|c|c}
    \toprule
     \textbf{Task} & \begin{tabular}[c]{@{}c@{}}\textbf{Length,}\\ \textbf{tokens} \end{tabular} & \begin{tabular}[c]{@{}c@{}}\textbf{LLama-3.2-1B,}\\ \textbf{ARMT}\\ \end{tabular}  & \begin{tabular}[c]{@{}c@{}}\textbf{LLama-3.2-1B,}\\ \textbf{ARMT,}\\ \textbf{Diagonal Batching}\\ \end{tabular} & \textbf{Speed-up ($\times$ times)} \\
     \midrule
    \multirow[c]{6}{*}{QA1} & 2K & 13.43 & 15.06 & 0.89 \\
    &  4K & 22.45 & 17.99 & 1.25 \\
    &  8K & 41.41 & 22.49 & 1.84 \\
    &  16K & 79.16 & 33.12 & 2.39 \\
    &  32K & 153.68 & 54.20 & 2.84 \\
    &  64K & 302.15 & 94.36 & 3.20 \\
    \midrule
    \multirow[c]{6}{*}{QA2} & 2K & 13.08 & 14.93 & 0.88 \\
    & 4K & 22.66 & 18.21 & 1.24 \\
    & 8K & 41.66 & 22.70 & 1.84 \\
    & 16K & 79.80 & 33.38 & 2.39 \\
    & 32K & 153.82 & 53.46 & 2.88 \\
    & 64K & 303.40 & 94.69 & 3.20 \\
    \bottomrule
    \end{tabular}%
  }
  \caption{Diagonal Batching significantly speeds up ARMT inference on longer inputs. Inference time (in seconds) and relative speed-up of the models are given on the BABILong dataset, first two tasks.}
  \label{tab:babilong_time}
\end{table}


Finally, we implemented backward pass for diagonal batching to support training. Aligning the training and inference code eliminates a discrepancy that is likely the source of logit-level floating-point drift.

\section{Additional measurements}

To clearly illustrate the speedup provided by the developed diagonal batching algorithm, we present relative improvements across various configurations and sequence lengths. Results for speedup against original ARMT implementation is shown in \Cref{tab:speedup_over_armt_llama1b} and against underlying Llama model in \Cref{tab:speedup_over_llama1b}. These measurements provide additional insights into how our method scales and compares to the baseline implementations.

We also present results for different size models of Llama-3 family~\cite{grattafiori2024llama}: LLaMA-160M (\Cref{tab:perf_comparison_llama160m}), 1B (\Cref{tab:merged_perf_speedup_llama1b}), 3B (\Cref{tab:perf_comparison_llama3b}), and 8B (\Cref{tab:perf_comparison_llama8b}) models.

\begin{table}[h]
  \centering
  \resizebox{\textwidth}{!}{%
  \renewcommand{\arraystretch}{1.2}
  \begin{tabular}{l*{6}{l}}
  \toprule
  \textbf{Method} & \multicolumn{6}{c}{\textbf{Sequence Length}} \\
  \cmidrule(lr){2-7}
   & {\textbf{4096}} & {\textbf{8192}} & {\textbf{16384}} & {\textbf{32768}} & {\textbf{65536}} & {\textbf{131072}} \\
  \midrule
  Llama-3.2-3B & 0.168 & 0.344 & 0.769 & 1.95 & 5.59 & 18.2 \\
  \rowcolor{gray!10} \textbf{Configuration: (1024, 128)} \\
  LLama-3.2-3B-ARMT & 0.272 & 0.537 & 1.05 & 2.02 & 4.09 & 8.23 \\
  Diagonal Batching: LLama-3.1-3B-ARMT & 0.274 \textcolor{darkred}{x0.99} & 0.454 \textcolor{teal}{x1.18} & 0.833 \textcolor{teal}{x1.26} & 1.58 \textcolor{teal}{x1.28} & 3.1 \textcolor{teal}{x1.32} & 6.14 \textcolor{teal}{x1.34} \\
  \rowcolor{gray!10} \textbf{Configuration: (4096, 128)} \\
  LLama-3.2-3B-ARMT & 0.203 & 0.39 & 0.765 & 1.52 & 3.01 & 6.01 \\
  Diagonal Batching: LLama-3.2-3B-ARMT & 0.239 \textcolor{darkred}{x0.85} & 0.411 \textcolor{darkred}{x0.95} & 0.739 \textcolor{teal}{x1.04} & 1.4 \textcolor{teal}{x1.09} & 2.72 \textcolor{teal}{x1.11} & 5.37 \textcolor{teal}{x1.12} \\
  \midrule
  \end{tabular}%
  }
  \caption{Diagonal batching speed-ups the execution - from 1.1 to 1.3 times comparing to base ARMT for 131072 sequence length.  Execution time comparison (in seconds) and relative speedups across different sequence lengths compared to LLama-3.2-3B-ARMT. Configuration in format (segment\_size, memory\_tokens). Nvidia A100 GPU.}
  \label{tab:perf_comparison_llama3b}
\end{table}
\begin{table}[h]
  \centering
  \renewcommand{\arraystretch}{1.2}
  \resizebox{\textwidth}{!}{%
  \begin{tabular}{l*{6}{l}}
  \toprule
  \textbf{Method} & \multicolumn{6}{c}{\textbf{Sequence Length}} \\
  \cmidrule(lr){2-7}
   & {\textbf{4096}} & {\textbf{8192}} & {\textbf{16384}} & {\textbf{32768}} & {\textbf{65536}} & {\textbf{131072}} \\
  \midrule
  Llama-3.1-8B & 0.332 & 0.682 & 1.48 & 3.61 & 9.82 & 30.4 \\
  \rowcolor{gray!10} \textbf{Configuration: (1024, 128)} \\
  LLama-3.1-8B-ARMT & 0.497 & 0.936 & 1.82 & 3.63 & 7.22 & 14.4 \\
  Diagonal Batching: LLama-3.1-8B-ARMT & 0.478 \textcolor{teal}{x1.04} & 0.86 \textcolor{teal}{x1.09} & 1.64 \textcolor{teal}{x1.11} & 3.2 \textcolor{teal}{x1.13} & 6.34 \textcolor{teal}{x1.14} & 12.6 \textcolor{teal}{x1.14} \\
  \rowcolor{gray!10} \textbf{Configuration: (4096, 128)} \\
  LLama-3.1-8B-ARMT& 0.384 & 0.754 & 1.48 & 2.95 & 5.86 & 11.7 \\
  Diagonal Batching: LLama-3.1-8B-ARMT & 0.432 \textcolor{darkred}{x0.89} & 0.781 \textcolor{darkred}{x0.97} & 1.46 \textcolor{teal}{x1.01} & 2.83 \textcolor{teal}{x1.04} & 5.6 \textcolor{teal}{x1.05} & 11.1 \textcolor{teal}{x1.05} \\
  \midrule
  \end{tabular}%
  }
  \caption{Diagonal batching speed-ups the execution - from 1.05 to 1.14 times comparing to base ARMT for 131072 sequence length. Execution time comparison (in seconds) and relative speedups across different sequence lengths compared to LLama-3.2-8B-ARMT. Configuration in format (segment\_size, memory\_tokens). Nvidia A100 GPU.}
  \label{tab:perf_comparison_llama8b}
\end{table}

\begin{table}[h]
  \centering
  \renewcommand{\arraystretch}{1.2}
  \resizebox{\textwidth}{!}{%
  \begin{tabular}{l*{6}{l}}
  \toprule
  \textbf{Method} & \multicolumn{6}{c}{\textbf{Sequence Length}} \\
  \cmidrule(lr){2-7}
   & {\textbf{4096}} & {\textbf{8192}} & {\textbf{16384}} & {\textbf{32768}} & {\textbf{65536}} & {\textbf{131072}} \\
  \midrule
  Llama-160M & 0.017 & 0.033 & 0.075 & 0.196 & 0.594 & 2.03 \\
  \rowcolor{gray!10} \textbf{Configuration: (1024, 128)} \\
  LLama-160M-ARMT & 0.105 & 0.211 & 0.422 & 0.877 & 1.72 & 3.37 \\
  Diagonal Batching: LLama-160M-ARMT & 0.061 \textcolor{teal}{x1.72} & 0.087 \textcolor{teal}{x2.43} & 0.138 \textcolor{teal}{x3.06} & 0.243 \textcolor{teal}{x3.61} & 0.451 \textcolor{teal}{x3.81} & 0.855 \textcolor{teal}{x3.94} \\
  \rowcolor{gray!10} \textbf{Configuration: (4096, 128)} \\
  LLama-160M-ARMT & 0.031 & 0.057 & 0.111 & 0.216 & 0.432 & 0.855 \\
  Diagonal Batching: LLama-160M-ARMT & 0.046 \textcolor{darkred}{x0.67} & 0.062 \textcolor{darkred}{x0.92} & 0.094 \textcolor{teal}{x1.18} & 0.156 \textcolor{teal}{x1.38} & 0.284 \textcolor{teal}{x1.52} & 0.537 \textcolor{teal}{x1.59} \\
  \midrule
  \end{tabular}%
  }
  \caption{Diagonal batching speed-ups the execution - from 1.6 to 3.9 times comparing to base ARMT for 131072 sequence length. Execution time comparison (in seconds) and relative speedups across different sequence lengths compared to LLama-160M-ARMT. Configuration in format (segment\_size, memory\_tokens). Nvidia A100 GPU.}
  \label{tab:perf_comparison_llama160m}
\end{table}

\begin{table}[H]
  \centering
  \renewcommand{\arraystretch}{1.2}
  \resizebox{\textwidth}{!}{%
  \begin{tabular}{l*{6}{S[table-format=3.3]}}
  \toprule
  \textbf{Method} & \multicolumn{6}{c}{\textbf{Sequence Length}} \\
  \cmidrule(lr){2-7}
   & {\textbf{4096}} & {\textbf{8192}} & {\textbf{16384}} & {\textbf{32768}} & {\textbf{65536}} & {\textbf{131072}} \\
  \midrule
  LLama-3.2-1B, configuration: (512, 128) & 0.085 & 0.105 & 0.828 & 1.075 & 1.473 & 2.473 \\
  \midrule
  LLama-3.2-1B, configuration: (1024, 128) & 0.202 & 0.133 & 1.071 & 1.412 & 1.937 & 3.290 \\
  \midrule
  LLama-3.2-1B, configuration: (2048, 128) & 0.222 & 0.148 & 1.237 & 1.622 & 2.216 & 3.743 \\
  \midrule
  LLama-3.2-1B, configuration: (4096, 128) & 0.235 & 0.151 & 1.275 & 1.675 & 2.299 & 3.886 \\
  \bottomrule
  \end{tabular}%
  }
  \caption{Diagonal batching ARMT implementation allows to speedup the execution for longer sequences due to  linear complexity - from 2.4 times to 3.8 times with respect to LLama-3.2-1B for 131072 sequence length. Table shows Diagonal Batching executor speedup against original LLama-3.2-1B for different methods across sequence lengths. Configuration in format (segment\_size, memory\_tokens). Measured on Nvidia A100 GPU.}
  \label{tab:speedup_over_llama1b}
\end{table}

\begin{table}[H]
  \centering
  \renewcommand{\arraystretch}{1.2}
  \resizebox{\textwidth}{!}{%
  \begin{tabular}{l*{6}{S[table-format=3.3]}}
  \toprule
  \textbf{Method} & \multicolumn{6}{c}{\textbf{Sequence Length}} \\
  \cmidrule(lr){2-7}
   & {\textbf{4096}} & {\textbf{8192}} & {\textbf{16384}} & {\textbf{32768}} & {\textbf{65536}} & {\textbf{131072}} \\
  \midrule
  LLama-3.2-1B, configuration: (512, 128) & 0.519 & 2.315 & 2.533 & 2.660 & 2.707 & 2.721 \\
  \midrule
  LLama-3.2-1B, configuration: (1024, 128) & 1.252 & 1.485 & 1.647 & 1.753 & 1.811 & 1.806 \\
  \midrule
  LLama-3.2-1B, configuration: (2048, 128) & 0.870 & 1.006 & 1.132 & 1.189 & 1.216 & 1.229 \\
  \midrule
  LLama-3.2-1B, configuration: (4096, 128) & 0.804 & 0.901 & 1.020 & 1.074 & 1.103 & 1.119 \\
  \bottomrule
  \end{tabular}%
  }
  \caption{Diagonal batching allows to speedup the execution for longer sequences - from 1.1 times to 2.7 times with respect to base ARMT for 131072 sequence length. In cases when diagonal batching is slower, we can fall back to the original inference algorithm at runtime. Table shows Diagonal Batching executor speedup against original ARMT inplementation for different methods across sequence lengths. Configuration in format (segment\_size, memory\_tokens). Measured on Nvidia A100 GPU.}
  \label{tab:speedup_over_armt_llama1b}
\end{table}

\end{document}